\documentclass[runningheads]{llncs}
\usepackage{graphicx}
\usepackage{amsmath}
\usepackage{amssymb}
\usepackage{upgreek}
\usepackage{xcolor}
\usepackage{booktabs}
\usepackage{todonotes}
\usepackage{multirow} 
\usepackage{times}
\usepackage{tikz}
\usetikzlibrary{arrows,automata}
\usepackage{url}

\usepackage{algorithm}  
\usepackage{algorithmicx}
\usepackage[noend]{algpseudocode}

\algrenewcommand{\algorithmiccomment}[1]{\hfill\(\triangleright\) \textcolor[gray]{0.5}{#1}}
\newcommand{\posi}{C_{\ge 0}}
\newcommand{\nega}{C_{\le 0}}
\newcommand{\fsharp}[1]{f^\#(#1)}

\newcommand{\robcomp}[1]{y_{C_f(x)}- y_{#1} \le 0}

\newtheorem{invariant}[theorem]{Invariant}
\begin{document}
\title{Improving Neural Network Verification through Spurious Region Guided Refinement}
%
%
\author{Pengfei Yang\inst{1,2} \and Renjue Li\inst{1,2}\and Jianlin Li\inst{1,2}\and Cheng-Chao Huang\inst{3}\and Jingyi~Wang\inst{4}\and Jun~Sun\inst{5}\and Bai Xue\inst{1,2}\and Lijun Zhang\inst{1,2,3}}
\authorrunning{P. Yang et al.}
%
\institute{SKLCS, Institute of Software, Chinese Academy of Sciences, Beijing, China \and
University of Chinese Academy of Sciences, Beijing, China
 \and
Institute of Intelligent Software, Guangzhou, China \and
Zhejiang University, Hangzhou, China \and
Singapore Management University, Singapore
}
\maketitle              
\begin{abstract}
We propose a spurious region guided refinement approach for robustness verification of deep neural networks. Our method starts with applying the DeepPoly abstract domain to analyze the network. If the robustness property cannot be verified, the result is inconclusive. Due to the over-approximation, the computed region in the abstraction  may be \emph{spurious} in the sense that it does not contain any true counterexample. Our goal is to identify such spurious regions and use them to guide the abstraction refinement. The core idea is to make use of the obtained  constraints of the abstraction to infer new bounds for the neurons. This is achieved by linear programming techniques.  With the new bounds, we iteratively apply DeepPoly, aiming to  eliminate spurious regions. We have implemented our approach in a prototypical tool DeepSRGR. Experimental results show 
that a large amount of regions can be identified as spurious, 
and as a result, the precision of DeepPoly can be significantly improved. As a side contribution, we show that our approach can be applied to verify quantitative robustness properties.
\end{abstract}

\section{Introduction}
In recent years, deep neural networks (DNNs) have achieved exceptional performance in many applications. 
They are often applied to perform tasks which are particularly challenging for traditional logic-based software, e.g., 
nature language processing~\cite{DBLP:journals/spm/X12a}, image classification~\cite{DBLP:conf/nips/KrizhevskySH12}, and game playing~\cite{alphago}.
Unfortunately, DNNs have also been shown to be often lack of robustness and vulnerable to adversarial samples~\cite{SZSBEGF2014}, i.e., it is possible to add a small (and even imperceptible) perturbation to a correctly classified input so that it is mis-classified by a well-trained DNN. 
This raises concerns on deploying DNNs in safety-critical applications like self-driving cars~\cite{selfdriving}, medical systems~\cite{medicalsystem}, and malware detection~\cite{DBLP:journals/access/LiXCH19}. 
It is thus important that robustness of DNNs is verified before they are deployed in safety-critical domains.

In this work, we focus on (local) robustness, i.e., given an input and a manipulation region around the input (which is usually specified according to a certain norm), we verify that a given DNN never makes any mistake on any input in the region. 
The first work on DNN verification was published in~\cite{DBLP:conf/cav/PulinaT10}, 
which focuses on DNNs with sigmoid activation functions with a partition-refinement approach. 
In 2017, Katz et al.~\cite{reluplex} and Ehlers~\cite{planet} independently implemented Reluplex and Planet, two SMT solvers to verify DNNs with the $\mathrm{ReLU}$ activation function on properties expressible with SMT constraints. 
Since 2018, abstract interpretation has been one of the most popular methods for DNN verification in the lead of AI${}^2$~\cite{AI2}, and subsequent works like \cite{deepz,deeppoly,deepsymbol,charon,krelu,deeppolygpu} have improved AI${}^2$ in terms of efficiency, precision and more activation functions (like sigmoid and $\tanh$) so that abstract interpretation based approach can be applied to DNNs of larger size and more complex structures.

Among the above methods, DeepPoly~\cite{deeppoly} is a most outstanding one regarding precision and scalability.
DeepPoly is an abstract domain specially developed for DNN verification.
It sufficiently considers the structures and the operators of a DNN , and it designs a polytope expression which not only fits for these structures and operators to control the loss of precision, but also works with a very small time overhead to achieve scalability.
However, as an abstraction interpretation based method, it provides very little insight if it fails to verify the property.
In this work, we propose a method to improve DeepPoly by eliminating spurious regions through abstraction refinement.
A spurious region is a region computed using abstract semantics, conjuncted with the negation of the property to be verified. 
This region is spurious in the sense that if the property is satisfied, then this region, although not empty, does not contain any true counterexample which can be realized in the original program.
In this case, we propose a refinement strategy to rule out the spurious region, i.e., to prove that this region does not contain any true counterexamples.

Our approach is based on DeepPoly and improves it by refinement of the spurious region through linear programming. 
The core idea is to intersect the abstraction constructed by abstract interpretation with the negation of the property to generate a spurious region, and perform linear programming on the constraints of the spurious region so that the bounds of the $\mathrm{ReLU}$ neurons whose behaviors are uncertain can be tightened. 
As a result, some of these neurons can be determined to be definitely activated or deactivated, which significantly improves the precision of the abstraction given by abstract interpretation.
This procedure can be performed iteratively and the precision of the abstraction are gradually improved, so that we are likely to rule out this spurious region in some iteration.
If we successfully rule out all the possible spurious regions through such an iterative refinement, the property is soundly verified.
Our method is similar in spirit to counterexample guided abstraction refinement (CEGAR)~\cite{cegar00}, i.e., we apply abstract interpretation for abstraction and linear programming for refinement. A fundamental difference is that we use the constraints of the spurious region, instead of a concrete counterexample (which is challenging to construct in our setting), as the guidance of refinement.

The same spurious region guided refinement approach is also effective in quantitative robustness verification. 
Instead of requiring that all inputs in the region should be correctly classified, a certain probability of error in the region is allowed.
Quantitative robustness is more realistic and general compared to the ordinary robustness, and a DNN verified against quantitative robustness is useful in practice as well. 
The spurious region guided refinement approach naturally fits for this setting, since a comparatively precise over-approximation of the spurious region implies a sound robustness confidence.
To the best of our knowledge, this is the first work to verify quantitative robustness with strict soundness guarantee, which distinguishes our approach from the previous sampling based methods like \cite{statistical19,DBLP:conf/icml/WengCNSBOD19,DBLP:journals/corr/abs-2002-06864}.

In summary, our main contributions are as follows:
\begin{itemize}
    \item We propose spurious region guided refinement to verify robustness properties of deep neural networks. This approach significantly improves the precision of DeepPoly and it can verify more challenging properties than DeepPoly.
    \item We implement the algorithms as a prototype and run them on networks trained on popular datasets like MNIST and ACAS Xu. 
The experimental results show that our approach significantly improves the precision of DeepPoly in successfully verifying much stronger robustness properties  (larger maximum radius) and determining the behaviors of a great proportion of uncertain $\mathrm{ReLU}$ neurons.
    \item We apply our approach to solve quantitative robustness verification problem with strict soundness guarantee. In the experiments, we observe that, comparing to using only DeepPoly, the bounds by our approach can be up to two orders of magnitudes better in the experiments.
\end{itemize}

\noindent\emph{Organisations of the paper.} We provide preliminaries in Section~\ref{sec:pre}. DeepPoly is recalled in Section~\ref{subsec:deeppoly}. We present our overall verification framework and the algorithm in Section~\ref{sec:algorithm}, and discuss quantitative robustness verification in Section~\ref{sec:qrv}. Section~\ref{sec:eval} evaluates our algorithms through experiments. Section~\ref{sec:conclusion} reviews related work and concludes the paper.

\section{Preliminaries}\label{sec:pre}
In this section we recall some basic notions on deep neural networks, local robustness verification, and abstract interpretation. Given a vector $x \in \mathbb R^n$, we write $x_i$ to denote its $i$-th entry for $1 \le i \le n$.

\subsection{Robustness verification of deep neural networks}
In this work, we focus on deep feedforward neural networks (DNNs), which can be represented as a function $f:\mathbb{R}^m \to \mathbb{R}^n$, mapping an input $ x \in \mathbb{R}^m$ to its output $y =  f(x)   \in \mathbb{R}^n$.
A DNN $f$ often classifies an input $x$ by obtaining the maximum dimension of the output, i.e., $\arg \max_{1 \le i \le n} f(x)_i$. We denote such a DNN by $C_f:\mathbb{R}^m \to C$ which is defined by $C_f(x)=\arg \max_{1 \le i \le n} f(x)_i$ where $C=\{1,\ldots,n\}$ is the set of classification classes.

A DNN has a sequence of layers, including an input layer at the beginning, followed by several hidden layers, and an output layer in the end.
The output of a layer is the input of the next layer. Each layer contains multiple neurons, the number of which is known as the dimension of the layer. 
The DNN $f$ is the composition of the transformations between layers.
Typically an affine transformation followed by a non-linear activation function is performed.
For an affine transformation $y=Ax+b$, if the matrix $A$ is not sparse, we call such a layer fully connected.
A DNN with only fully connected layers and activation functions is a fully connected neural network (FNN).
In this work, we focus on the rectified linear unit (ReLU) activation function, defined as
$\mathrm{ReLU}(x)=\max(x,0)$
for $x \in \mathbb R$. 
Typically, a DNN verification problem is defined as follows:
\begin{definition}
\label{def:DNNverification}
Given a DNN $f:\mathbb R^m \to \mathbb R^n$, a set of inputs $X \subseteq \mathbb R^m$, and a property $P \subseteq \mathbb R^n$, we need to determine whether $f(X):=\{f(x) \mid x \in X\} \subseteq P$ holds.
\end{definition}

Local robustness describes the stability of the behaviour of a normal input under a perturbation. 
The range of input under this perturbation is the robustness region. 
For a DNN $C_f(x)$ which performs classification tasks, a robustness property  typically states that $C_f$ outputs the same class on the robustness region.

There are various ways to define a robustness region, and one of the most popular ways is to use the $L_p$ norm.
For $x \in \mathbb{R}^m$ and $1 \le p < \infty$, we define the $L_p$ norm of $x$ to be $\| x\|_p=\left(\sum_{i=1}^m |x_i|^p \right)^{\frac 1p},$
and its $L_\infty$ norm $\| x\|_\infty=\max_{1 \le i \le m} |x_i|.$
We write $\bar B_p(x,r):=\{x' \in \mathbb{R}^m \mid \|x-x'\|_p \le r\}$ to represent a (closed) $L_p$ ball for $x \in \mathbb{R}^m$ and $r>0$, which is a neighbourhood of $x$ as its robustness region.
If we set $X=\bar B_p(x,r)$ and $P=\{y \in \mathbb R^n \mid  \arg \max_i y_i=C_{f}(x)\}$ in Def.~\ref{def:DNNverification}, it is exactly the robustness verification problem. 
Hereafter, we set $p=\infty$.

\subsection{Abstract interpretation for DNN verification}

Abstract interpretation \cite{CC1977} is a static analysis method and it is aimed to find an over-approximation of the semantics of programs so as to verify their correctness. 
Generally we have a function $f:\mathbb R^m \to \mathbb R^n$ representing the concrete program, a set $X \subseteq \mathbb R^m$ representing the property that the input of the program satisfies, and a set $P \subseteq \mathbb R^n$ representing the property to verify. 
The problem is to determine whether $f(X) \subseteq P$ holds. However, if $f$ and $X$ are complex, it is difficult to calculate $f(X)$ and to determine whether $f(X) \subseteq P$ holds. 
Abstract interpretation uses abstract domains and abstract transformations to over-approximate sets and functions so that an over-approximation of the output can be obtained efficiently.

Now we have a concrete domain $\mathcal C$, which includes a set of inputs $X$ as one of its elements. To make computation efficient,
we need an abstract domain $\mathcal A$ to abstract the elements in the concrete domain. We assume that there is a partial order $\le$ on $\mathcal C$ as well as $\mathcal A$, which in our settings is the subset relation $\subseteq$.
\begin{definition}
A pair of functions $\alpha:{\mathcal C} \to {\mathcal A}$ and $\gamma:{\mathcal A} \to {\mathcal C}$ is a Galois connection, if for any $a \in {\mathcal A}$ and $c \in {\mathcal C}$, we have $\alpha(c) \le a \Leftrightarrow c \le \gamma(a)$.
\end{definition}
Intuitively, a Galois connection $(\alpha,\gamma)$ gives abstraction and concretization relations between two domains, respectively. 
Naturally $a \in \mathcal A$ is a sound abstraction of $c \in \mathcal C$ if and only if $c \le \gamma(a)$.

In abstract interpretation, it is important to choose a suitable abstract domain because it determines the efficiency and precision.
In practice, we use a certain type of constraints to represent the abstract elements in an abstract domain.
Classical abstract domains for Euclid spaces include Box, Zonotope~\cite{ghorbal2009zonotope,zonotope10}, and Polyhedra~\cite{fastpolyhedra17}.

Not only do we need abstract domains to over-approximate sets, but we are also required to adapt over-approximation to functions. Here we consider the lifting of the function $f:\mathbb R^m \to \mathbb R^n$ defined as $T_f(X):\mathcal P(\mathbb R^m) \to \mathcal P(\mathbb R^n)$, $T_f(X):=f(X)=\{f(x)\mid x \in X\}$. Now we have an abstract domain $\mathcal A_k$ for the $k$-dimension Euclid space and the corresponding concretization $\gamma$, and a function $T_f^\#:\mathcal A_m \to \mathcal A_n$ is a sound abstract transformer, if $T_f \circ \gamma \subseteq \gamma \circ T_f^\#$.

When we have a sound abstract $X^\# \in \mathcal A$ of $X$ and a sound abstract transformer $T_f^\#$, we can use the concretization of $T_f^\#(X^\#)$ to over-approximate $f(X)$ since we have $f(X)=T_f(X) \subseteq T_f(\gamma(X^\#)) \subseteq \gamma \circ T_f^\#(X^\#)$. If $\gamma \circ T_f^\#(X^\#) \subseteq P$, the property $P$ is successfully verified. Obviously, verification through abstract interpretation is sound but not complete.

AI${}^2$~\cite{AI2} first adopted abstract interpretation to verify DNNs, and many subsequent works like \cite{deepz,deeppoly,deepsymbol} focus on improving its efficiency and precision through, e.g., defining new abstract domains. 
As a deep neural network, the function $f:\mathbb R^m \to \mathbb R^n$ can be regarded as a composition $f=f_l \circ \cdots \circ f_1$ of its $l+1$ layers, where $f_j$ performs the transformation between the $j$-th and the $(j+1)$-th layer, i.e. it can be a linear transformation, or a $\mathrm{ReLU}$ operation. 

If we choose Box, Zonotope, or Polyhedra as the abstract domain, then for linear transformations and the $\mathrm{ReLU}$ function, their abstract transformers have been developed in \cite{AI2}. After we have abstract transformers $f_j^\#$ for these $f_j$, we can conduct abstract interpretation layer by layer as $f_l^\# \circ \cdots \circ f_1^\# (X^\#)$. 

\section{A Brief Introduction to DeepPoly}\label{subsec:deeppoly}

Our approach relies on the abstract domain DeepPoly~\cite{deeppoly}, which is the state-of-the-art  abstract domain for DNN verification. It defines the abstract transformers of multiple activation functions and layers used in DNNs. 
The core idea of DeepPoly is to give every variable an upper and a lower bound in the form of an affine expression using only variables that appear before it. It can express a polyhedron globally. Moreover,  experimentally, it often has better precision than  Box and Zonotope domains.

We denote the $n$-dimensional DeepPoly abstract domain with $\mathcal A_n$. 
Formally an abstract element $a \in \mathcal A_n$ is a tuple $(a^\le,a^\ge,l,u)$, where $a^\le$ and $a^\ge$ give the $i$-th variable $x_i$ a lower bound and an upper bound, respectively, in the form of a linear combination of variables which appear before it, i.e. $\sum_{k=1}^{i-1} w_kx_k + w_0$, for $i=1,\ldots,n$, and $l,u \in \mathbb R^n$ give the lower bound and upper bound of each variable, respectively. 
The concretization of $a$ is defined as
\begin{align} \label{eq:deeppolyconcrete}
\gamma(a)=\{x \in \mathbb R^n \mid a_i^\le \le x_i \le a_i^\ge, \enspace i=1,\ldots,n\}.
\end{align}
The abstract domain $\mathcal A_n$ also requests that its abstract elements $a$ should satisfy the invariant $\gamma(a) \subseteq [l,u]$. 
This invariant helps construct efficient abstract transformers.

For an affine transformation $x_i=\sum_{k=1}^{i-1} w_kx_k + w_0$, we set $a_i^\le=a_i^\ge=\sum_{k=1}^{i-1} w_kx_k + w_0$. 
By substituting the variables $x_j$ appearing in $a_i^\le$ with $a_j^\le$ or $a_j^\ge$ according to its coefficient at most $i-1$ times, we can obtain a sound lower bound in the form of linear combination on  input variables only, and $l_i$ can be computed immediately from the range of input variables. A similar procedure also works for computing $u_i$.

For a $\mathrm{ReLU}$ transformation $x_i=\mathrm{ReLU}(x_j)$, we consider two cases:
\begin{itemize}
    \item If $l_j \ge 0$ or $u_j \le 0$, this  $\mathrm{ReLU}$ neuron is definitely \emph{activated} or \emph{deactivated}, respectively. 
In this case, this $\mathrm{ReLU}$ transformation actually performs an affine transformation, and thus its abstract transformer can be defined as above.
    \item If $l_j < 0$ and $u_j > 0$, the behavior of this  $\mathrm{ReLU}$ neuron is \emph{uncertain}, and we need to over-approximate this relation with a linear upper/lower bound. 
The best upper bound is $a_i^\ge=\frac {u_j(x_j-l_j)}{u_j-l_j}$. 
For the lower bound, there are multiple choices $a_i^\le=\lambda x_j$ where $\lambda \in [0,1]$. We choose $\lambda \in \{0,1\}$ which minimizes the area of the constraints. Basically we have two abstraction modes here, corresponding to the two choices of $\lambda$.
\end{itemize}
Note that for a DNN with only $\mathrm{ReLU}$ as non-linear operators, over-approximation occurs only when there are uncertain $\mathrm{ReLU}$ neurons, which are over-approximated using a triangle. The key of improving the precision is thus to compute the bounds of the uncertain $\mathrm{ReLU}$ neurons as precisely as possible, and to determine the behaviors of the most uncertain $\mathrm{ReLU}$ neurons.

DeepPoly also supports activation functions which are monotonically increasing, convex on $(-\infty,0]$ and concave on $[0,+\infty)$, like sigmoid and $\tanh$, and it supports max pooling layers. Readers can refer to \cite{deeppoly} for details.

\section{Spurious Region Guided Refinement}
\label{sec:algorithm}

\begin{figure}[t]
\centering
\includegraphics[width=0.8\linewidth]{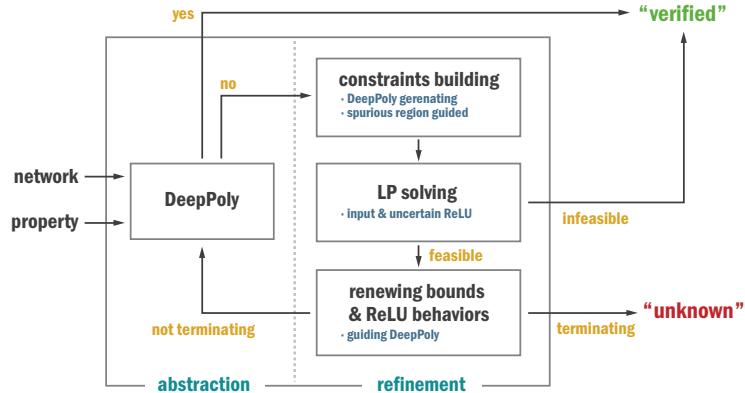} 
\caption{Framework of spurious region guided refinement} \label{fig:algflow}
\end{figure}

We explain the main steps of our algorithm, as depicted in Fig.~\ref{fig:algflow}. 
For the input property and network, we first employ DeepPoly as the initial step to compute
 $f^\#(X^\#)$. 
The concretization of $f^\#(X^\#)$ is the conjunction of many linear inequities given in Eq.~\ref{eq:deeppolyconcrete}, and for the robustness property $P$, the negation $\neg P$ is the disjunction of several linear inequities $\neg P =\bigvee_{t \ne C_f(x)} (y_{C_f(x)}-y_t \le 0)$. 
\begin{enumerate}
\item We check whether  $f^\#(X^\#) \cap^\# (y_{C_f(x)}-y_t \le 0) = \bot$ holds for each $t$. In case of yes, it indicates that the label $t$ cannot be classified, as it is dominated by $C_f(x)$. 
Otherwise,
we have $f^\#(X^\#) \cap^\# \neg P \ne \bot$, we have the conjunction $\gamma(f^\#(X^\#)) \wedge \neg P$ as a potential \emph{spurious region}, which represents the intersection of the abstraction of the real semantics and the negation of the property to verify.
We call such a region spurious because if the property is satisfied, then this region does not contain a true counterexample, i.e., a pair of input and output $(x^*,y^*)$ such that $y^*=f(x^*)$ and $y^*$ violates the property $P$.
In this case, this region is spuriously constructed due to the abstraction of the real semantics, where the counterexamples cannot be realized, and thus we aim to rule out the spurious region.

\item If no potential spurious region is found, our algorithm safely returns yes. 
\item 
Assume now that we have  a the potential spurious region.
The core idea is to use the constraints of the spurious region to refine this spurious region.
Here a natural way to refine the spurious region is linear programming, since all the constraints here are linear inequities. If the linear programming is infeasible, it indicates that the region is spurious, and thus we can return an affirmative result. Otherwise, 
our refinement will tighten the bounds of variables involved in the DNN, especially the input variables and uncertain $\mathrm{ReLU}$ neurons, and these tightened bounds help further  give a more precise abstraction.
\item As our approach is based on DeepPoly, similarly, we cannot guarantee completeness. We set a threshold $N$ of the number of iterations as a simple termination condition. If the termination condition is not reached, we run DeepPoly again, and return to the first step.
\end{enumerate}

Below we give an example, illustrating how refinement can help in robustness verification.




\begin{example}\label{example:algorithm}
Consider the network $f(x)=\mathrm{ReLU}\left(\begin{pmatrix}
1 & -1 \\
1 & 1
\end{pmatrix} x+\begin{pmatrix}
0 \\
2.5
\end{pmatrix}\right)$ and the region $\bar B_\infty((0,0)^\mathrm{T},1)$.
The robustness property $P$ here is $y_2-y_1 > 0$.
We invoke  first  DeepPoly: the lower bound  of $y_2-y_1$ given by DeepPoly is $-0.5$. 
As a result,
 the robustness property cannot be verified directly.
Fig.~\ref{fig:example}(a) shows  details of the example.
\end{example}

We fail to verify the property in Example~\ref{example:algorithm} because for the uncertain $\mathrm{ReLU}$ relation $y_1=\mathrm{ReLU}(x_3)$, the abstraction is imprecise, and the key to making the abstraction more precise here is to obtain as tight a bound as possible for $x_3$.

\begin{example}\label{example:algorithm2}
We use the constraints in Fig.~\ref{fig:example}(a) and additionally the constraint $y_2 - y_1 \le 0$ (i.e., $\neg P$) as the input of linear programming. Our aim is to obtain a tighter bound of the input neurons $x_1$ and $x_2$, as well as the uncertain $\mathrm{ReLU}$ neuron $x_3$, so the objective functions of the linear programming are $\min x_i$ and $\min -x_i$ for $i=1,2,3$.
All the three neurons have a tighter bound after the linear programming (see the red part in Fig.~\ref{fig:example}(b)). Fig.~\ref{fig:example}(b) shows the running of DeepPoly under these new bounds, where the input range and the abstraction of the uncertain  $\mathrm{ReLU}$ neuron are both refined.
Now the lower bound of $y_2-y_1$ is $0.25$, so DeepPoly successfully verifies the property.
\end{example}

\begin{figure}[t]
  \centering
  \scalebox{0.82}{
 \begin{tikzpicture}[->,>=stealth,auto,node distance=1.2cm,semithick,scale=1,every node/.style={scale=1}]
	\tikzstyle{blackdot}=[circle,fill=black,minimum size=6pt,inner sep=0pt]
	\tikzstyle{state}=[minimum size=0pt,circle,draw,thick]
	\tikzstyle{stateNframe}=[minimum size=0pt]	
	\node[state](x1){$x_1$};
	\node[state](x2)[below of=x1,yshift=-0.3cm]{$x_2$};
	\node[state](x3)[right of=x1,xshift=1.2cm]{$x_3$};
	\node[state](x4)[right of=x2,xshift=1.2cm]{$x_4$};
	\node[state](y1)[right of=x3,xshift=1.2cm]{$y_1$};
	\node[state](y2)[right of=x4,xshift=1.2cm]{$y_2$};
\node[stateNframe](f1)[above of=x1,yshift=-0.5cm]{$u_1=1$};	
\node[stateNframe](f2)[above of=f1,yshift=-0.85cm]{$l_1=-1$};	
\node[stateNframe](f3)[above of=f2,yshift=-0.85cm]{$x_1\le 1$};	
\node[stateNframe](f4)[above of=f3,yshift=-0.85cm]{$x_1 \ge -1$};

\node[stateNframe](g1)[above of=x3,yshift=-0.5cm]{$u_3=2$};	
\node[stateNframe](g2)[above of=g1,yshift=-0.85cm]{$l_3=-2$};	
\node[stateNframe](g3)[above of=g2,yshift=-0.85cm]{$x_3\le x_1-x_2$};	
\node[stateNframe](g4)[above of=g3,yshift=-0.85cm]{$x_3 \ge x_1-x_2$};

\node[stateNframe](h1)[above of=y1,yshift=-0.5cm]{$u_5=2$};	
\node[stateNframe](h2)[above of=h1,yshift=-0.85cm]{$l_5=0$};	
\node[stateNframe](h3)[above of=h2,yshift=-0.85cm]{$y_1\le 0.5 x_3+1$};	
\node[stateNframe](h4)[above of=h3,yshift=-0.85cm]{$y_1\ge 0$};

\node[stateNframe](i1)[below of=x2,yshift=0.5cm]{$x_2\ge -1$};	
\node[stateNframe](i2)[below of=i1,yshift=0.85cm]{$x_2 \le 1$};	
\node[stateNframe](i3)[below of=i2,yshift=0.85cm]{$l_2= -1$};	
\node[stateNframe](i4)[below of=i3,yshift=0.85cm]{$u_2 = 1$};

\node[stateNframe](j1)[below of=x4,yshift=0.5cm]{$x_4\ge x_1+x_2+2.5$};	
\node[stateNframe](j2)[below of=j1,yshift=0.85cm]{$x_4 \le x_1+x_2+2.5$};
\node[stateNframe](j3)[below of=j2,yshift=0.85cm]{$l_4= 0.5$};	
\node[stateNframe](j4)[below of=j3,yshift=0.85cm]{$u_4 = 4.5$};

\node[stateNframe](k1)[below of=y2,yshift=0.5cm]{$y_2\ge x_4$};	
\node[stateNframe](k2)[below of=k1,yshift=0.85cm]{$y_2 \le x_4$};
\node[stateNframe](k3)[below of=k2,yshift=0.85cm]{$l_6= 0.5$};	
\node[stateNframe](k4)[below of=k3,yshift=0.85cm]{$u_6 = 4.5$};

\node[stateNframe](a)[below of=j4,yshift=0.5cm]{(a)};

\node[state](x11)[right of=y1,xshift=1cm]{$x_1$};
	\node[state](x21)[below of=x11,yshift=-0.3cm]{$x_2$};
	\node[state](x31)[right of=x11,xshift=1.2cm]{$x_3$};
	\node[state](x41)[right of=x21,xshift=1.2cm]{$x_4$};
	\node[state](y11)[right of=x31,xshift=1.2cm]{$y_1$};
	\node[state](y21)[right of=x41,xshift=1.2cm]{$y_2$};
	
\node[stateNframe](l1)[above of=x11,yshift=-0.5cm]{\color{red}$u_1=0$};	
\node[stateNframe](l2)[above of=l1,yshift=-0.85cm]{$l_1=-1$};	
\node[stateNframe](l3)[above of=l2,yshift=-0.85cm]{\color{red}$x_1\le 0$};	
\node[stateNframe](l4)[above of=l3,yshift=-0.85cm]{$x_1 \ge -1$};

\node[stateNframe](m1)[below of=x21,yshift=0.5cm]{$x_2\ge -1$};	
\node[stateNframe](m2)[below of=m1,yshift=0.85cm]{\color{red}$x_2 \le -0.667$};	
\node[stateNframe](m3)[below of=m2,yshift=0.85cm]{$l_2= -1$};	
\node[stateNframe](m4)[below of=m3,yshift=0.85cm]{\color{red}$u_2 = -0.667$};

\node[stateNframe](n1)[above of=x31,yshift=-0.5cm]{\color{red}$u_3=1$};	
\node[stateNframe](n2)[above of=n1,yshift=-0.85cm]{\color{red}$l_3=-0.333$};	
\node[stateNframe](n3)[above of=n2,yshift=-0.85cm]{$x_3\le x_1-x_2$};	
\node[stateNframe](n4)[above of=n3,yshift=-0.85cm]{$x_3 \ge x_1-x_2$};

\node[stateNframe](o1)[below of=x41,yshift=0.5cm]{$x_4\ge x_1+x_2+2.5$};	
\node[stateNframe](o2)[below of=o1,yshift=0.85cm]{$x_4 \le x_1+x_2+2.5$};
\node[stateNframe](o3)[below of=o2,yshift=0.85cm]{$l_4= 0.5$};	
\node[stateNframe](o4)[below of=o3,yshift=0.85cm]{\color{blue}$u_4 = 1.833$};

\node[stateNframe](q1)[above of=y11,yshift=-0.5cm]{\color{blue}$u_5=1$};	
\node[stateNframe](q2)[above of=q1,yshift=-0.85cm]{$l_5=0$};	
\node[stateNframe](q3)[above of=q2,yshift=-0.85cm]{\enspace\enspace\color{red}$y_1\le 0.75 x_3+0.25$};	
\node[stateNframe](q4)[above of=q3,yshift=-0.85cm]{\color{red}$y_1\ge x_3$};

\node[stateNframe](p1)[below of=y21,yshift=0.5cm]{$y_2\ge x_4$};	
\node[stateNframe](p2)[below of=p1,yshift=0.85cm]{$y_2 \le x_4$};
\node[stateNframe](p3)[below of=p2,yshift=0.85cm]{$l_6= 0.5$};	
\node[stateNframe](p4)[below of=p3,yshift=0.85cm]{\color{blue}$u_6 = 1.833$};

\node[stateNframe](b)[below of=o4,yshift=0.5cm]{(b)};

	\path (x1) edge	[-]						node {$1$} (x3)
		(x1) edge[-]							node[xshift=-0.6cm,yshift=0.2cm] {$1$} (x4)
		(x2) edge	[-]						node[xshift=-0.2cm,yshift=-0.35cm] {$-1$} (x3)
		(x2) edge[-]							node {$1$} (x4)
			(x3) edge	[-]						node { $\mathrm{ReLU}(x_3)$} (y1)
		(x4) edge[-]							node {$\mathrm{ReLU}(x_4)$} (y2)
		(x11) edge	[-]						node {$1$} (x31)
		(x11) edge[-]							node[xshift=-0.6cm,yshift=0.2cm] {$1$} (x41)
		(x21) edge	[-]						node[xshift=-0.2cm,yshift=-0.35cm] {$-1$} (x31)
		(x21) edge[-]							node {$1$} (x41)
			(x31) edge	[-]						node { $\mathrm{ReLU}(x_3)$} (y11)
		(x41) edge[-]							node {$\mathrm{ReLU}(x_4)$} (y21)
		
			  ;
\end{tikzpicture}}
  \caption{\label{fig:example}
Example~\ref{example:algorithm} (left) and Example~\ref{example:algorithm2} (right): where the red parts are introduced through linear programming based refinement and the blue parts are introduced by a second run of DeepPoly.}
\end{figure}
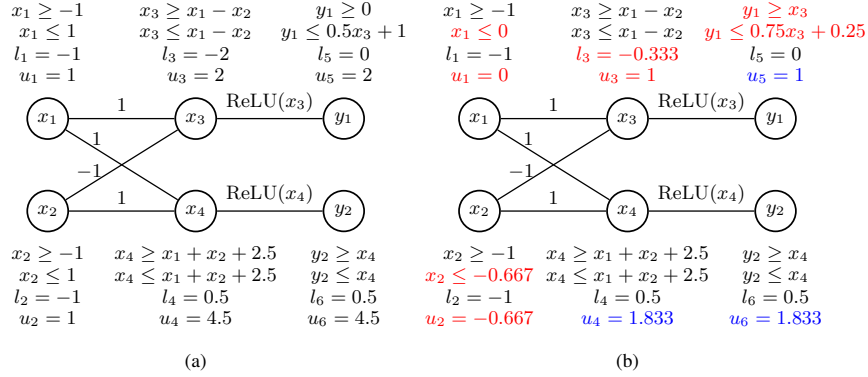


\subsection{Main algorithm}

\begin{algorithm}[t]
    \caption{Spurious region guided robustness verification}
    \label{Alg:main}
    \begin{algorithmic}[1]
        \Require
        \Statex  DNN $f$, input $x$, radius $r$.
        \Ensure
        \Statex Return ``YES'' if verified, or ``UNKNOWN'' otherwise.
        \Function{Verify}{$f$, $x$, $r$}
            \State $Y_0 \gets \fsharp{\bar B_\infty(x,r)}$ \label{line:initialai} \Comment{abstract interpretation with DeepPoly}
            \State $V_\mathrm{u} \gets \{v \mid v \text{ was marked as uncertain in Line~\ref{line:initialai}}\}$
            \State $A = \{ t \mid Y_0 \cap^\# ( \robcomp{t} ) \ne \bot \}$ \label{line:findspu}
            \If{$A = \emptyset$} \Return YES \Comment{otherwise $A = \{t_1, \ldots, t_l \}$} \EndIf
            \For{$i \gets 1$ to $l$}
                \State Verified $\gets $ False, $V \gets V_\mathrm{u}$, $Y \gets Y_0$ \Comment{denote $Y=(Y^\le,Y^\ge,l,u)$}
                \State $C_{\ge 0} \gets \emptyset$, $C_{\le 0} \gets \emptyset$ \Comment{set of new activated/deactivated neurons}
                \State $\mathrm{Spu} \gets ( \robcomp{t_i} ) \wedge \bigwedge_{j=1}^{i-1} ( y_{C_f(x)}-y_{t_j} \ge 0) $ \Comment{spurious region}
                \While{terminating condition not satisfied} \label{line:terminating}
                    \If{$Y \wedge \mathrm{Spu}$ is infeasible} 
                        \State Verified $\gets $ True
                        \State \textbf{break} 
                    \EndIf
                    \For{$v \in V \cup V_0$}
                        \Comment{$V_0$: set of input neurons}
                        \State $(l_v,u_v) \gets \Call{LP}{Y \wedge \mathrm{Spu}, v}$
                    \EndFor
                    \For{$v \in V$}
                        \If{$l_v \ge 0$} 
                            \State $\posi \gets \posi \cup \{v\}$, $V \gets V \setminus \{v\}$
                        \ElsIf{$u_v \le 0$}
                            \State $\nega \gets \nega \cup \{v\}$,
                            $V \gets V \setminus \{v\}$
                        \EndIf
                    \EndFor
                    \State $X \gets
                            \bigcap_{v \in V_0} 
                            \{ l_v \le v \le u_v\}$
                    \State $Y \gets \fsharp{X}$ according to $\posi$, $\nega$, $l$, and $u$ 
                   \State $V \gets \{v \mid v \text{ was marked as uncertain in Line~22}\} \setminus (C_{\ge 0} \cup C_{\le 0})$
                    \If{$Y\cap^\# ( \robcomp{t_i} ) = \bot$}
                    \State Verified $\gets $ True
                        \State \textbf{break} 
                    \EndIf
                   \EndWhile
                    \If {Verified $=$ False} \Return UNKNOWN \EndIf
            \EndFor
            \State \Return YES
        \EndFunction
\end{algorithmic}\end{algorithm}

Alg.~\ref{Alg:main} presents our algorithm. 
First we run abstract interpretation to find the uncertain neurons and the spurious regions (Line~{2--5}). 
For each possible spurious region, we have a \textbf{while} loop which iteratively refines the abstraction. In each iteration we perform linear programming to renew the bounds of the input neurons and uncertain $\mathrm{ReLU}$ neurons; when we find that the bound of an uncertain $\mathrm{ReLU}$ neuron becomes definitely non-negative or non-positive, then the $\mathrm{ReLU}$ behavior of this neuron is renewed (Line~{14--20}). We use them to guide abstract interpretation in the next step (Line~{21--22}). 
Here in Line~22, we make sure that during the abstract interpretation, the abstraction of previous uncertain neurons (namely the uncertain neurons before the linear programming step in the same iteration) compulsorily follows the new bounds and new $\mathrm{ReLU}$ behaviors given by the current $C_{\ge 0}$, $C_{\le 0}$, $l$, and $u$, where these bounds will not be renewed by abstract interpretation, and the concretization of $Y$ is defined as
\begin{align}\label{eq:concrete2}
    \gamma(Y)=\{ x \mid \forall i. \enspace Y_i^\le \le x_i \le Y_i^\ge\} \cap [l,u].
\end{align}
The \textbf{while} loop ends when (i) either we find that the spurious region is infeasible (Line~{11, 24}) and we proceed to refine the next spurious region, with a label Verified True, (ii) or we reach the terminating condition and fail to rule out this spurious region, in which case we return UNKNOWN. 
If every \textbf{while} loop ends with the label Verified True, we successfully rule out all the spurious regions and return YES. 
An observation is that, if some spurious regions have been ruled out, we can add the constraints of their negation to make the current spurious region smaller so as to improve the precision (Line~9).

Here we discuss the soundness of Alg.~\ref{Alg:main}. We focus on the \textbf{while} loop and claim that it has the following loop invariant: 
\begin{invariant}\label{inv:soundness}
The abstract element $Y$ over-approximates the intersection of the semantics of $f$ on $\bar B_\infty(x,r)$ and the spurious region, i.e., $f(\bar B_\infty(x,r)) \cap \mathrm{Spu} \subseteq \gamma(Y)$.
\end{invariant}
The initialization of $Y$ is $f^\#(\bar B_\infty(x,r))$ and it is naturally an over-approximation. The box $X$ is obtained by linear programming on $Y \wedge \mathrm{Spu}$, and $f^\#(X)$ is calculated through abstract interpretation and the bounds given by linear programming on $Y \wedge \mathrm{Spu}$, and thus it remains an over-approximation. It is worth mentioning that, when we run DeepPoly in Line~22, we are using the bounds obtained by linear programming to guide DeepPoly, and this may violate the invariant $\gamma(a) \subseteq [l,u]$ mentioned in Sect.~\ref{subsec:deeppoly}. Nonotheless, soundness still holds since the concretization of $Y$ is newly defined in Eq.~\ref{eq:concrete2}, where both items in the intersection over-approximate $f(\bar B_\infty(x,r)) \cap \mathrm{Spu}$. With Invarient~\ref{inv:soundness}, Alg.~\ref{Alg:main} returns YES if for any possible spurious region $\mathrm{Spu}$, the over-approximation of $f(\bar B_\infty(x,r)) \cap \mathrm{Spu}$ is infeasible, which implies the soundness of Alg.~\ref{Alg:main}.

\subsection{Iterative refinement of the spurious region}
\label{subsec:theory}
Here we present more theoretical insight on the iterative refinement of the spurious region. 
An iteration of the \textbf{while} loop in Alg.~\ref{Alg:main} can be represented as a function $\mathcal L:\mathcal A \to \mathcal A$, where $\mathcal A$ is the DeepPoly domain. 
An interesting observation is that, the abstract transformer $f^\#$ in the DeepPoly domain is not necessarily increasing, because different input ranges, even if they have inclusion relation, may lead to different choices of the abstraction mode of some uncertain $\mathrm{ReLU}$ neurons, which may violate the inclusion relation of abstraction. We have found such examples during our experiment, which is illustrated in the following example.
\begin{example}
Let $f(x)=\mathrm{ReLU}(x)$ with input ranges $I_1=[-2,1]$ and $I_2=[-2,3]$. We have $f^\#(I_1)=\{(x_1,x_2)^\mathrm{T} \in \mathbb R^2 \mid -2 \le x_1 \le 1,\enspace x_2 \ge 0, \enspace x_2 \le \frac 13 x_1+\frac 23\}$ and $f^\#(I_2)=\{(x_1,x_2)^\mathrm{T} \in \mathbb R^2 \mid -2 \le x_1 \le 3,\enspace x_2 \ge x_1,\enspace x_2 \le \frac 35 x_1+\frac 65\}$. We observe $(1,0)^\mathrm{T} \in f^\#(I_1)$ but $(1,0)^\mathrm{T} \notin f^\#(I_2)$, which implies that the transformer $f^\#$ is not increasing.
\end{example}
This fact also implies that $\mathcal{L}$ is not necessarily increasing, which violates the condition of Kleene's Theorem on fixed point \cite{kleene}.

Now we turn to the analysis of the sequence $\{Y_k=\mathcal L ^k (f^\#(\bar B_\infty(x,r)))\}_{k=1}^\infty$, where $\mathcal L^1:=\mathcal L$ and $\mathcal L^k:=\mathcal L \circ \mathcal L^{k-1}$ for $k \ge 2$. First we have the following lemma showing that in our settings every decreasing chain $S$ in the DeepPoly domain $\mathcal A$ has a meet $\bigcap^\# S \in \mathcal A$.

\begin{lemma} \label{lemma:CPO}
Let $\mathcal A_n$ be the $n$-dimensional DeepPoly domain and $\{a^{(k)}\} \subseteq \mathcal A_n$ a decreasing bounded sequence of non-empty abstract elements. If the coefficients in $a _i^{(k),\le}$ and $a_i^{(k),\ge}$ are uniformly bounded, then there exists an abstract element $a^* \in \mathcal A_n$ s.t. $\gamma(a^*)=\bigcap_{k=1}^\infty \gamma(a^{(k)})$.
\end{lemma}
\textbf{Remark:} The condition that the coefficients in $a _i^{(k),\le}$ and $a_i^{(k),\ge}$ are uniformly bounded are naturally satisfied in our setting, since in a DNN the coefficients and bounds involved have only finitely many values. Readers can refer to Appendix for a formal proof.

Lemma~\ref{lemma:CPO} implies that if our sequence $\{Y_k\}$ is decreasing, then the iterative refinement converges to an abstract element in DeepPoly, which is the greatest fixed point of $\mathcal L$ that is smaller than $f^\#(\bar B_\infty(x,r))$. A sufficient condition for $\{Y_k\}$ being decreasing is that during the abstract interpretation in every $Y_k$, every initial uncertain neuron maintains its abstraction mode, i.e. its corresponding $\lambda$ does not change, before its $\mathrm{ReLU}$ behavior is determined. A weaker sufficient condition for convergence is that change in abstraction mode of uncertain neurons never happens after finitely many iterations. 

If the abstraction mode of uncertain neurons changes infinitely often, generally the sequence $\{Y_k\}$ does not converge. In this case, we can consider its subsequence in which every $Y_k$ is obtained with the same abstraction mode. It is easy to see that such a subsequence must be decreasing and thus have a meet, as it is an accumulative point of the sequence $\{Y_k\}$. Since there are only finitely many choices of abstraction modes, such a accumulative points exists in $\{Y_k\}$, and there are only  finitely many accumulative points. We conclude these results in the following theorem which describes the convergence behavior of our iterative refinement of the spurious region:

\begin{theorem}
There exists a subsequence $\{Y_{n_k}\}$ of $\{Y_k\}$ s.t. $\{Y_{n_k}\}$ is decreasing and thus has a meet $\bigcap^\# \{Y_{n_k}\}$. Moreover, the set 
$$
\left\{\bigcap{}^\# \{Y_{n_k} \} \mid \{Y_{n_k}\} \text{ is a decreasing subsequence of } \{Y_k\}\right\}
$$
is finite, and it is a singleton if exact one abstraction mode of uncertain $\mathrm{ReLU}$ neurons happens infinitely often. 
\end{theorem}
\begin{proof}
Since the abstraction modes of uncertain $\mathrm{ReLU}$ neurons have only finitely many choices, there must be one which happens infinitely often in the computation of the sequence $\{Y_k\}$, and we choose the subsequence $\{Y_{n_k}\}$ in which every item is computed through this abstraction mode. Obviously $\{Y_{n_k}\}$ is decreasing and thus has a meet.

For a decreasing subsequence $\{Y_{n_k}\}$, we can find its subsequnce in which the abstraction mode of uncertain $\mathrm{ReLU}$ neurons does not change, and they have the same meet. Since there are only finitely many choices of abstraction modes of uncertain $\mathrm{ReLU}$ neurons, such accumulative points of $\{Y_k\}$ also have finitely many values.
If exact one abstraction mode of uncertain $\mathrm{ReLU}$ neurons happens infinitely often, obviously there is only one accumulative point in $\{Y_k\}$.\hfill \qed
\end{proof}

\subsection{Optimizations} \label{subsec:optimizations}
In the implementation of our main algorithm, we propose the following optimizations to improve the precision of refinement.

\paragraph{Optimization 1: More precise constraints in linear programming.} 
In Line~15 of Alg.~\ref{Alg:main}, it is not the best choice to take the linear constraints in the abstract element $Y$ into linear programming, because the abstraction of uncertain $\mathrm{ReLU}$ neurons in DeepPoly is not the best. Planet~\cite{planet} has a component which gives a more precise linear approximation for uncertain $\mathrm{ReLU}$ relations, where it uses the linear constraints
$y \le \frac {u(x-l)}{u-l},\enspace y \ge x,\enspace y \ge 0$
to over-approximate the relation $y=\mathrm{ReLU}(x)$ with $x \in [l,u]$.

\paragraph{Optimization 2: A better choice of the spurious region.}
If a true counterexample exists, there must exist an input $x' \in \bar B_\infty(x,r)$ s.t. $C_f(x),t \in \arg \max_i f(x')_i$ with some $t \ne C_f(x)$ since $f$ is continuous and $\bar B_\infty(x,r)$ is convex. That is to say, $y_{C_f(x)}=y_t$ is a necessary condition for the existence of a true counterexample, and we can choose $(y_{C_f(x)}-y_{t_i}=0)\wedge \bigwedge_{j=1}^{i-1} ( y_{C_f(x)}-y_{t_j} \ge 0)$ as the spurious region in Line~9 of Alg.~\ref{Alg:main}. This optimization makes the spurious region even smaller and benefits the precision improvement.

\paragraph{Optimization 3: Priority to work on small spurious regions.}
In Line~6 of Alg.~\ref{Alg:main},we determine the order of refining the spurious regions based on their sizes, i.e., a smaller region is chosen earlier.  This is based on the intuition that Alg.~\ref{Alg:main} works effectively if the spurious region is small. After the small spurious regions are ruled out, the constraints of large spurious regions can be tightened with the conjunction $\bigwedge_{j=1}^{i-1}(y_{C_f(x)}-y_{t_j} \ge 0)$. It is difficult to strictly determine which spurious region is the smallest, and thus we refer to the lower bound of $y_{C_f(x)}-y_{t_i}$ given by DeepPoly, i.e., the larger this lower bound is, the smaller the spurious region is likely to be, and we perform the \textbf{for} loop in Line~6 of Alg.~\ref{Alg:main} in this order. It is worth mentioning that, this optimization still makes sense even if we already adopt Optimization 2, since intuitively a larger spurious region $(y_{C_f(x)}-y_{t_i}\le 0)\wedge \bigwedge_{j=1}^{i-1} ( y_{C_f(x)}-y_{t_j} \ge 0)$ is more likely to have a larger boundary $(y_{C_f(x)}-y_{t_i}=0)\wedge \bigwedge_{j=1}^{i-1} ( y_{C_f(x)}-y_{t_j} \ge 0)$.


\section{Quantitative Robustness Verification}\label{sec:qrv}
In this section we recall the notion of quantitative robustness and show how to verify a quantitative robustness property of a DNN with spurious region guided refinement.

In practice, we may not need a strict condition of robustness to ensure that an input $x$ is not an adversarial example. A notion of mutation testing is proposed in \cite{detection1,detection2}, which requires that an input $x$ is normal if it has a low \emph{label change rate} on its neighbourhood. 
They follow a statistical way to estimate the label change rate of an input, which motivates us to give a formal definition of the property showing a low label change rate, and to consider the verification problem for such a property. Below we recall 
 the definition of   \emph{quantitative robustness} \cite{DBLP:journals/corr/abs-1902-05983},
 where we have a parameter $0< \eta \le 1$ representing the confidence of robustness.

\begin{definition}\label{def:approrobustness}
Given a DNN $C_f:\mathbb{R}^m \to C$, an input $x \in \mathbb R ^m$,  $r>0$, $0< \eta \le 1$, and a probability measure $\mu$ on $\bar B_\infty(x,r)$, $f$ is $\eta$-robust at $x$, if 
\begin{align*}
\mu(\{x' \in \bar  B_\infty(x,r) \mid C_f(x')=C_f(x)\}) \ge \eta.
\end{align*}
\end{definition}
Def.~\ref{def:approrobustness} has a tight association with label change rate, i.e., if $x$ is $\eta$-robust, then the label change rate should be larger than, or close to $1-\eta$. 
Hereafter, we set $\mu$ to be the uniform distribution on $\bar  B_\infty(x,r)$.

It is natural to adapt spurious region guided refinement to quantitative robustness verification. 
In Alg.~\ref{Alg:main}, we do not return UNKNOWN when we cannot rule out a spurious region, but record the volume of the box $X$ as an over-approximation of the Lebesgue measure of the spurious region.
After we work on all the spurious regions, we calculate the sum of these volume, and obtain a sound robustness confidence. Here we do not calculate the volume of the spurious region because precise calculation of volume of a high-dimensional polytope remains open, and we do not choose to use randomized algorithms because it may not be sound.

We further improve the algorithm through the powerset technique~\cite{AI2}.
Powerset technique is a classical and effective way to enhance the precision of abstract interpretation. 
Basically we split the input region into several subsets, and run abstract interpretation on these subsets, 
In our quantitative robustness verification setting, powerset technique not only improves the precision, but also accelerates the algorithm in some situations:  
If the subsets have the same volume, and the percentage of the subsets on which we may fail to verify robustness is already smaller than $1-\eta$, then we have successfully verified the $\eta$-robustness property.





\section{Experimental Evaluation}\label{sec:eval}

We implement our approach as a prototype called DeepSRGR. The implementation is based on a re-implementation of  the ReLU and the affine abstract transformers of DeepPoly in Python 3.7 and we amend it accordingly to implement Alg.~\ref{Alg:main}. 
We use CVXPY~\cite{diamond2016cvxpy} as our modeling language for convex optimization problems and CBC~\cite{CBC} as the LP solver. 
It is worth mentioning that we ignore the floating point error in our re-implementation of DeepPoly because sound linear programming currently does not scale in our experiments.
In the terminating condition, we set $N=5$. 
All the experiments adopt Optimization~1 and Optimization~3 in Sect.~\ref{subsec:optimizations}. All the experiments are conducted on a CentOS 7.7 server with 16 Intel Xeon Plwatinum 8153 @2.00GHz (16 cores) and 512G RAM, and they use 96 sub-processes concurrently at most. Readers can find all the source code and other experimental materials in \url{https://github.com/CAS-LRJ/RefineRobustness}.

\paragraph{Datasets.}
We use MNIST~\cite{L1998Gradient} and ACAS Xu ~\cite{acasxu14,acasxu15}
as the datasets in our experiments. 
MNIST contains $60\,000$ grayscale handwritten digits of the size $28 \times 28$. We can train DNNs to classify the images by the written digits on them.
The ACAS Xu system is aimed to avoid airborne collisions for unmanned aircrafts and it uses an observation table to make decisions for the aircraft. In \cite{DBLP:journals/corr/abs-1810-04240}, the observation table is realized by training DNNs instead of storing it.

\paragraph{Networks.} On MNIST, we trained seven fully connected networks of the size $6 \times 20$, $3 \times 50$, $3 \times 100$, $6 \times 100$, $6 \times 200$, $9 \times 200$, and $6 \times 500$, where $m \times n$ refers $m$ hidden layers and $n$ neurons in each hidden layer, and we name them from FNN2 to FNN8, respectively (we also have a small network FNN1 for testing). On ACAS Xu, we randomly choose three networks used in~\cite{reluplex}, all of the size $5 \times 50$.

\subsection{Improvement in precision}
First we compare DeepPoly and DeepSRGR in terms of their precision of robustness verification. We consider the following two indices: (i) the maximum radius that the two tools can verify, and (ii) the number of uncertain $\mathrm{ReLU}$ neurons whose behaviors can be further determined by DeepSRGR. We randomly choose three images from the MNIST dataset, and calculate their maximum radius that the two tools can verify through a binary search on the seven FNNs we trained. We also record  the number of the uncertain $\mathrm{ReLU}$ neurons whose behaviors are renewed to definitely activated/deactivated on the maximum radius of DeepSRGR. We do not adopt Optimization 2 in Sect.~\ref{subsec:optimizations} in this experiment because Optimization 2 cannot be used in quantitative robustness verification, and we suppose that the evaluation of precision in this experiment holds for both verification tasks.

Table~\ref{tab:3picboundrate} shows the results. We can see from the table that DeepSRGR can verify stronger (i.e., larger maximum radius) robustness properties than DeepPoly, and determine behaviors of a large proportion of uncertain $\mathrm{ReLU}$ neurons even on these most challenging properties. The average number of iterations for ruling out a spurious region is around or below $3$ in all the running examples, and more than half of the spurious regions can be ruled out within $2$ iterations.

\begin{table}[t]
    \centering
    \scalebox{0.9}{
            \begin{tabular}{c|cc|c|cc|cc|cc}
            \toprule
            &
            \multicolumn{2}{|c|}{Maximum radius}&
            \multirow{2}{*}{\shortstack{\# spurious \\ regions}}&
            \multicolumn{2}{|c|}{\# uncertain $\mathrm{ReLU}$}&
            \multicolumn{2}{|c|}{\% renewed}&
            \multicolumn{2}{|c}{\# iterations}
            \\
            & DeepPoly & DeepSRGR & & Original & Renewed & {\enspace\enspace}MAX{\enspace\enspace} & {\enspace\enspace}AVG{\enspace\enspace} & MAX & GT \\ 
            \midrule
              & 0.034 & 0.047 & 6 &  51 &    38 & 74.5\% &    48.4\% &5 &17 \\
            FNN2  & 0.017 & 0.023 & 3 &  47 &    37 & 78.7\% &     51.8\% &4 &9 \\
             & 0.017 & 0.023 &1 &  34 &    25 & 73.5\% & 73.5\%  &4 &4  \\
             \midrule
            & 0.049 & 0.066 &6 & 88  &   69 & 78.4\% &  60.9\% &5 &15  \\
            FNN3  & 0.025 & 0.033 &7 & 94 &    85 & 90.4\% &    46.0\% &5 &18 \\
              & 0.045 & 0.058 &3 &  98 &    45 & 45.1\% &   27.2\% &5 &9 \\
              \midrule
             & 0.045 & 0.060 &6 &  180 &    102 & 56.7\% &   35.2\% & 5 & 19 \\
            FNN4  & 0.024 & 0.030 &6 &  199 &   144 & 72.4\% &   36.5\% &4 &15  \\
              &0.035         & 0.046   &   2  &   155  &   103       & 66.5\%       &   42.9\% &5 & 7 \\
              \midrule
              &0.034         & 0.042   &  7   &   305  &  245       & 80.3\%       &    37.8\% & 5 & 20  \\
             FNN5  &0.016         & 0.019     & 5   &   315  &     204     & 64.8\%       &    34.0\% & 4 & 14  \\
              &0.021        & 0.027    & 7    &   337  &   256       & 76.0\%       &   34.9\%  &5 &18 \\  
              \midrule
              &0.022 &0.026 &7 &683 &271 &39.7\% &19.8\% &4 &18\\
              FNN6 &0.011 &0.013 &6 &657 &483 &73.5\% &36.7\% &3 &14\\
              &0.021 &0.025 &8 &723 &169 &23.4\% &12.2\% &5 &21\\
                \midrule
              &0.021 &0.023 &9 &987 &297 &30.1\% &10.0\% &5 &29\\
              FNN7 &0.010 &0.011 &5 &877 &648 &73.9\% &26.8\% &3 &11\\
              &0.017 &0.019 &7 &913 &352 &38.6\% &24.3\% &3 &16\\
                \midrule
              &0.037 &0.044 &9 &1\,504 &976 &64.9\% &45.9\% &5 &36\\
              FNN8 &0.020 &0.022 &9 &1\,213 &818 &67.4\% &33.3\% &3 &21\\
              &0.033 &0.040 &9 &1\,371 &1\,269 &92.6\% &51.1\% &5 &37\\
            \bottomrule
            \end{tabular}
    }
    \caption{Maximum radius which can be verified by DeepPoly and DeepSRGR, and details of DeepSRGR running on its maximum radius, where in the number of renewed uncertain nuerons, we show the largest one among the spurious regions.}
    \label{tab:3picboundrate}
\end{table}

\subsection{Robustness verification performance}
We further evaluate our tool DeepSRGR by verifying more challenging robustness properties. We randomly choose $50$ samples from the MNIST dataset. On FNN4, FNN5, FNN6, and FNN7, we fix four radii, $0.037$, $0.026$, $0.021$, and $0.015$, for the four networks respectively, and verify the robustness property with the corresponding  radius on the $50$ inputs. The radius chosen here is very challenging for the corresponding  network. We adopt Optimization 2 in Sect.~\ref{subsec:optimizations} in this experiment.

Table~\ref{tab:batchverification} presents the results. DeepSRGR works significantly better than DeepPoly in verifying these  properties. Linear programming in DeepSRGR takes a large amount of time in the experiment, and thus DeepSRGR is less efficient.

\begin{table}[t]
    \centering
    \scalebox{1}{
            \begin{tabular}{ccccccc}
            \toprule
\multirow{2}{*}{Model} & \multirow{2}{*}{{\enspace\enspace\enspace\enspace}Size{\enspace\enspace\enspace\enspace}} & \multirow{2}{*}{{\enspace}Radius{\enspace}} & \multicolumn{2}{c}{\# verified} & \multicolumn{2}{c}{Time (s)} \\
                       &                       &                         & {\enspace}DeepPoly{\enspace}       & {\enspace}DeepSRGR{\enspace}       & {\enspace} MAX{\enspace}            & {\enspace} AVG{\enspace}           \\
                       \midrule
FNN4      &   $3 \times 100$       &    0.037      &   14    &  35  & 3\,384         &  781      \\
FNN5      &    $6 \times 100$      &    0.026      &   19    &  31  &  7\,508       &  1\,689      \\
FNN6      &     $6 \times 200$     &    0.021      &   14    &  25  &  23\,157       &   6\,178     \\
FNN7      &    $9 \times 200$      &     0.015     &    25   &  36  &  61\,760       &    8\,960    \\
                       \bottomrule
\end{tabular}
    }
    \caption{The number that DeepPoly and DeepSRGR verifies among the $50$ inputs, and the maximum/average running time of DeepSRGR.}
    \label{tab:batchverification}
\end{table}

Furthermore, we again run the $15$ running examples which is not verified by DeepSRGR on FNN4. This time we change the maximum number of iterations to $20$ and $50$, and obtain the following interesting observations:
\begin{itemize}
    \item Two more properties (out of $15$) are successfully verified when we change $N$ to $20$. No more properties can be verified even if we change $N$ from $20$ to $50$.
    \item In this experiments, $13$ more spurious regions are ruled out, six of which takes $6$ iterations, one takes $7$, two takes $8$, and the other four takes $13$, $22$, $27$, and $32$ iterations. In these running examples, the average number of renewed $\mathrm{ReLU}$ behaviors is $102.8$, and a large proportion are renewed in the last iteration ($47.4\%$ on average). Fig.~\ref{fig:exprenew} shows the detailed results.
    \item As for the $13$ spurious regions which cannot be ruled out within $50$ iterations, the average number of renewed $\mathrm{ReLU}$ behaviors is only $8.54$, which is significantly lower than the average of the $13$ spurious regions which are newly ruled out. In these running examples, changes in $\mathrm{ReLU}$ behaviors and  $\mathrm{ReLU}$ abstraction modes do not happen after the 9th iteration, and the average number is $4.4$. 
\end{itemize}

\begin{figure}[t]
\centering
\includegraphics[width=0.7\linewidth]{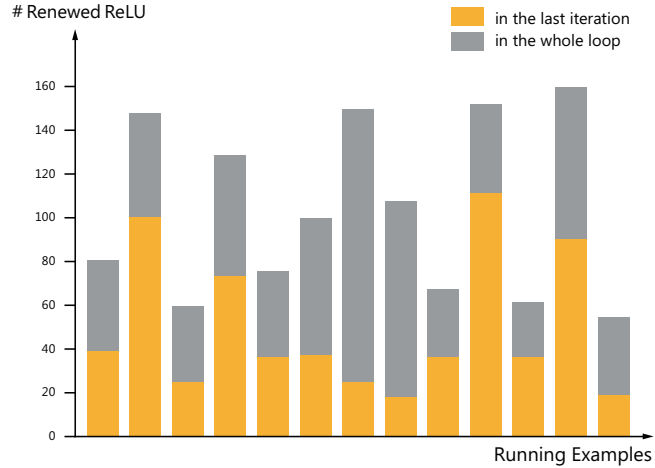} 
\caption{Number of renewed $\mathrm{ReLU}$ behaviors in the spurious regions newly ruled out.} \label{fig:exprenew}
\end{figure}

We observe that, by increasing  the termination threshold $N$ from $5$ to $50$, only two more properties out of $15$ can be verified additionally. This suggests that our method can effectively identify these spurious regions which are relevant to verification of the property, in a small number of iterations.

\subsection{Quantitative robustness verification on ACAS Xu networks}

\begin{figure}[t]
\centering
\includegraphics[width=0.7\linewidth]{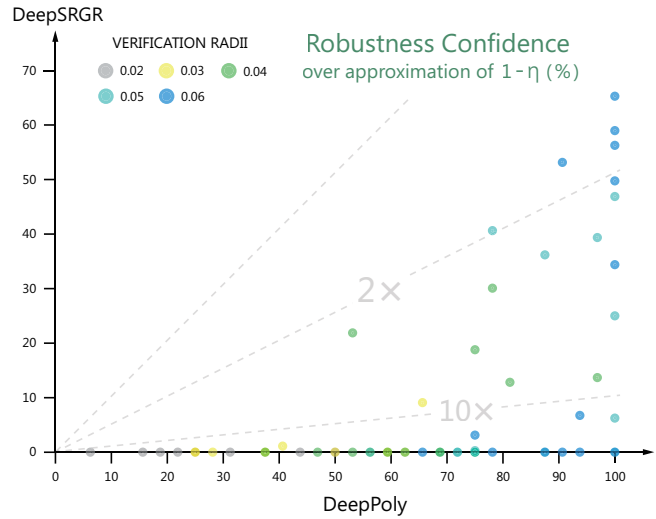} 
\caption{Quantitative robustness verification using DeepPoly and DeepSRGR} \label{fig:expquantitative}
\end{figure}

We evaluate DeepSRGR for quantitative robustness verification on ACAS Xu networks. 
We randomly choose five inputs, and compute the maximum robustness radius for each input on the three networks with DeepPoly through a binary search. 
In our experiment, the radius for a running example is the maximum robustness radius plus $0.02$, $0.03$, $0.04$, $0.05$, and $0.06$. 
We use the powerset technique and the number of splits is $32$.
For DeepPoly, the robustness confidence it gives is the proportion of the splits on which DeepPoly verifies the property.

Fig.~\ref{fig:expquantitative} shows the results. We can see that DeepSRGR gives significantly better over-approximation of $1-\eta$ than DeepPoly. That is, in more than $90\%$ running examples, our over-approximation is no more than one half of that given by DeepPoly, and in more than $75\%$, our over-approximation is even smaller than one tenth of that given by DeepPoly.

\section{Related Works and Conclusion}\label{sec:conclusion}
We have already discussed papers mostly related to our paper. Here we add some more new results.  Marabou~\cite{marabou} has been developed as the next generation of Reluplex. Recently, some verification approaches based on  abstractions of DNN models have been proposed in \cite{abstractioncav20,atva2020}. 
In addition, alternative approaches based on constraint-solving~\cite{LM2017,DBLP:conf/aaai/NarodytskaKRSW18,DBLP:journals/jmlr/BunelLTTKK20,DBLP:conf/cvpr/LinYCZLLH19}, layer-by-layer exhaustive search~\cite{DBLP:conf/cav/HuangKWW17}, global optimization~\cite{RHK2018,DBLP:conf/nfm/DuttaJST18,DBLP:conf/ijcai/RuanWSHKK19}, functional approximation~\cite{fastlin}, reduction to two-player games~\cite{DBLP:conf/tacas/WickerHK18,DBLP:journals/tcs/WuWRHK20}, and star set based abstraction~\cite{starset19,imagestar20} have been proposed as well.

In this work, we propose a spurious region guided refinement approach for robustness and quantitative robustness verification of deep neural networks, where abstract interpretation calculates an abstraction, and linear programming performs refinement with the guidance of the spurious region.
Our experimental results show that our tool can significantly improve the precision of DeepPoly, verify more robustness properties, and often provide a quantitative robustness with strict soundness guarantee.  

Abstraction interpretation based framework is quite extensive to different DNN models, different properties, and incorporate different verification methods. As future work, we will investigate how to increase the precision further by using more precise linear over-approximation like~\cite{krelu}.

\bibliographystyle{splncs04}
\bibliography{bib}

\begin{thebibliography}{10}
\providecommand{\url}[1]{\texttt{#1}}
\providecommand{\urlprefix}{URL }
\providecommand{\doi}[1]{https://doi.org/#1}

\bibitem{DBLP:journals/spm/X12a}
Deep neural networks for acoustic modeling in speech recognition: The shared
  views of four research groups. {IEEE} Signal Process. Mag.  \textbf{29}(6),
  82--97 (2012)

\bibitem{charon}
Anderson, G., Pailoor, S., Dillig, I., Chaudhuri, S.: Optimization and
  abstraction: a synergistic approach for analyzing neural network robustness.
  In: McKinley, K.S., Fisher, K. (eds.) Proceedings of the 40th {ACM} {SIGPLAN}
  Conference on Programming Language Design and Implementation, {PLDI} 2019,
  Phoenix, AZ, USA, June 22-26, 2019. pp. 731--744. {ACM} (2019)

\bibitem{DBLP:journals/corr/abs-2002-06864}
Baluta, T., Chua, Z.L., Meel, K.S., Saxena, P.: Scalable quantitative
  verification for deep neural networks. CoRR  \textbf{abs/2002.06864} (2020),
  \url{https://arxiv.org/abs/2002.06864}

\bibitem{kleene}
Baranga, A.: The contraction principle as a particular case of kleene's fixed
  point theorem. Discret. Math.  \textbf{98}(1),  75--79 (1991)

\bibitem{DBLP:journals/jmlr/BunelLTTKK20}
Bunel, R., Lu, J., Turkaslan, I., Torr, P.H.S., Kohli, P., Kumar, M.P.: Branch
  and bound for piecewise linear neural network verification. J. Mach. Learn.
  Res.  \textbf{21},  42:1--42:39 (2020)

\bibitem{cegar00}
Clarke, E.M., Grumberg, O., Jha, S., Lu, Y., Veith, H.: Counterexample-guided
  abstraction refinement. In: Emerson, E.A., Sistla, A.P. (eds.) Computer Aided
  Verification, 12th International Conference, {CAV} 2000, Chicago, IL, USA,
  July 15-19, 2000, Proceedings. Lecture Notes in Computer Science, vol.~1855,
  pp. 154--169. Springer (2000)

\bibitem{CC1977}
Cousot, P., Cousot, R.: Abstract interpretation: A unified lattice model for
  static analysis of programs by construction or approximation of fixpoints.
  In: Fourth ACM Symposium on Principles of Programming Languages (POPL). pp.
  238--252 (1977)

\bibitem{diamond2016cvxpy}
Diamond, S., Boyd, S.: {CVXPY}: {A} {P}ython-embedded modeling language for
  convex optimization. Journal of Machine Learning Research  \textbf{17}(83),
  ~1--5 (2016)

\bibitem{DBLP:conf/nfm/DuttaJST18}
Dutta, S., Jha, S., Sankaranarayanan, S., Tiwari, A.: Output range analysis for
  deep feedforward neural networks. In: Dutle, A., Mu{\~{n}}oz, C.A.,
  Narkawicz, A. (eds.) {NASA} Formal Methods - 10th International Symposium,
  {NFM} 2018, Newport News, VA, USA, April 17-19, 2018, Proceedings. Lecture
  Notes in Computer Science, vol. 10811, pp. 121--138. Springer (2018)

\bibitem{planet}
Ehlers, R.: Formal verification of piece-wise linear feed-forward neural
  networks. In: 15th International Symposium on Automated Technology for
  Verification and Analysis (ATVA2017). pp. 269--286 (2017)

\bibitem{abstractioncav20}
Elboher, Y.Y., Gottschlich, J., Katz, G.: An abstraction-based framework for
  neural network verification. In: Lahiri, S.K., Wang, C. (eds.) Computer Aided
  Verification - 32nd International Conference, {CAV} 2020, Los Angeles, CA,
  USA, July 21-24, 2020, Proceedings, Part {I}. Lecture Notes in Computer
  Science, vol. 12224, pp. 43--65. Springer (2020)

\bibitem{acasxu14}
von Essen, C., Giannakopoulou, D.: Analyzing the next generation airborne
  collision avoidance system. In: {\'{A}}brah{\'{a}}m, E., Havelund, K. (eds.)
  Tools and Algorithms for the Construction and Analysis of Systems - 20th
  International Conference, {TACAS} 2014, Held as Part of the European Joint
  Conferences on Theory and Practice of Software, {ETAPS} 2014, Grenoble,
  France, April 5-13, 2014. Proceedings. Lecture Notes in Computer Science,
  vol.~8413, pp. 620--635. Springer (2014)

\bibitem{AI2}
Gehr, T., Mirman, M., Drachsler-Cohen, D., Tsankov, P., Chaudhuri, S., Vechev,
  M.: {AI$^2$}: Safety and robustness certification of neural networks with
  abstract interpretation. In: 2018 IEEE Symposium on Security and Privacy
  (S\&P 2018). pp. 948--963 (2018)

\bibitem{ghorbal2009zonotope}
Ghorbal, K., Goubault, E., Putot, S.: The zonotope abstract domain taylor1+.
  In: International Conference on Computer Aided Verification. pp. 627--633.
  Springer (2009)

\bibitem{zonotope10}
Ghorbal, K., Goubault, E., Putot, S.: A logical product approach to zonotope
  intersection. In: Touili, T., Cook, B., Jackson, P.B. (eds.) Computer Aided
  Verification, 22nd International Conference, {CAV} 2010, Edinburgh, UK, July
  15-19, 2010. Proceedings. Lecture Notes in Computer Science, vol.~6174, pp.
  212--226. Springer (2010)

\bibitem{DBLP:conf/cav/HuangKWW17}
Huang, X., Kwiatkowska, M., Wang, S., Wu, M.: Safety verification of deep
  neural networks. In: 29th International Conference on Computer Aided
  Verification (CAV2017). pp. 3--29 (2017)

\bibitem{acasxu15}
Jeannin, J., Ghorbal, K., Kouskoulas, Y., Gardner, R., Schmidt, A., Zawadzki,
  E., Platzer, A.: Formal verification of {ACAS} x, an industrial airborne
  collision avoidance system. In: Girault, A., Guan, N. (eds.) 2015
  International Conference on Embedded Software, {EMSOFT} 2015, Amsterdam,
  Netherlands, October 4-9, 2015. pp. 127--136. {IEEE} (2015)

\bibitem{CBC}
johnjforrest, Vigerske, S., Santos, H.G., Ralphs, T., Hafer, L., Kristjansson,
  B., jpfasano, EdwinStraver, Lubin, M., rlougee, jpgoncal1, h-i gassmann,
  Saltzman, M.: coin-or/cbc: Version 2.10.5 (Mar 2020).
  \doi{10.5281/zenodo.3700700}, \url{https://doi.org/10.5281/zenodo.3700700}

\bibitem{DBLP:journals/corr/abs-1810-04240}
Julian, K.D., Kochenderfer, M.J., Owen, M.P.: Deep neural network compression
  for aircraft collision avoidance systems. CoRR  \textbf{abs/1810.04240}
  (2018), \url{http://arxiv.org/abs/1810.04240}

\bibitem{reluplex}
Katz, G., Barrett, C.W., Dill, D.L., Julian, K., Kochenderfer, M.J.: Reluplex:
  An efficient {SMT} solver for verifying deep neural networks. In: 29th
  International Conference on Computer Aided Verification (CAV2017). pp.
  97--117 (2017)

\bibitem{marabou}
Katz, G., Huang, D.A., Ibeling, D., Julian, K., Lazarus, C., Lim, R., Shah, P.,
  Thakoor, S., Wu, H., Zeljic, A., Dill, D.L., Kochenderfer, M.J., Barrett,
  C.W.: The marabou framework for verification and analysis of deep neural
  networks. In: Dillig, I., Tasiran, S. (eds.) Computer Aided Verification -
  31st International Conference, {CAV} 2019, New York City, NY, USA, July
  15-18, 2019, Proceedings, Part {I}. Lecture Notes in Computer Science, vol.
  11561, pp. 443--452. Springer (2019)

\bibitem{DBLP:conf/nips/KrizhevskySH12}
Krizhevsky, A., Sutskever, I., Hinton, G.E.: Imagenet classification with deep
  convolutional neural networks. In: Advances in Neural Information Processing
  Systems 25: 26th Annual Conference on Neural Information Processing Systems
  2012. Proceedings of a meeting held December 3-6, 2012, Lake Tahoe, Nevada,
  United States. pp. 1106--1114 (2012)

\bibitem{L1998Gradient}
L{\'{e}}cun, Y., Bottou, L., Bengio, Y., Haffner, P.: Gradient-based learning
  applied to document recognition. Proceedings of the IEEE  \textbf{86}(11),
  2278--2324 (1998)

\bibitem{deepsymbol}
Li, J., Liu, J., Yang, P., Chen, L., Huang, X., Zhang, L.: Analyzing deep
  neural networks with symbolic propagation: Towards higher precision and
  faster verification. In: Chang, B.E. (ed.) Static Analysis - 26th
  International Symposium, {SAS} 2019, Porto, Portugal, October 8-11, 2019,
  Proceedings. Lecture Notes in Computer Science, vol. 11822, pp. 296--319.
  Springer (2019)

\bibitem{DBLP:journals/access/LiXCH19}
Li, Y., Xiong, K., Chin, T., Hu, C.: A machine learning framework for domain
  generation algorithm-based malware detection. {IEEE} Access  \textbf{7},
  32765--32782 (2019)

\bibitem{DBLP:conf/cvpr/LinYCZLLH19}
Lin, W., Yang, Z., Chen, X., Zhao, Q., Li, X., Liu, Z., He, J.: Robustness
  verification of classification deep neural networks via linear programming.
  In: {IEEE} Conference on Computer Vision and Pattern Recognition, {CVPR}
  2019, Long Beach, CA, USA, June 16-20, 2019. pp. 11418--11427. Computer
  Vision Foundation / {IEEE} (2019)

\bibitem{LM2017}
Lomuscio, A., Maganti, L.: An approach to reachability analysis for
  feed-forward {ReLU} neural networks. In: KR2018 (2018)

\bibitem{DBLP:journals/corr/abs-1902-05983}
Mangal, R., Nori, A.V., Orso, A.: Robustness of neural networks: {A}
  probabilistic and practical approach. CoRR  \textbf{abs/1902.05983} (2019),
  \url{http://arxiv.org/abs/1902.05983}

\bibitem{deeppolygpu}
M{\"{u}}ller, C., Singh, G., P{\"{u}}schel, M., Vechev, M.T.: Neural network
  robustness verification on gpus. CoRR  \textbf{abs/2007.10868} (2020),
  \url{https://arxiv.org/abs/2007.10868}

\bibitem{DBLP:conf/aaai/NarodytskaKRSW18}
Narodytska, N., Kasiviswanathan, S.P., Ryzhyk, L., Sagiv, M., Walsh, T.:
  Verifying properties of binarized deep neural networks. In: McIlraith, S.A.,
  Weinberger, K.Q. (eds.) Proceedings of the Thirty-Second {AAAI} Conference on
  Artificial Intelligence, (AAAI-18), the 30th innovative Applications of
  Artificial Intelligence (IAAI-18), and the 8th {AAAI} Symposium on
  Educational Advances in Artificial Intelligence (EAAI-18), New Orleans,
  Louisiana, USA, February 2-7, 2018. pp. 6615--6624. {AAAI} Press (2018)

\bibitem{atva2020}
Pranav~Ashok, Vahid~Hashemi, J.K., M\"uhlberger, S.: Deepabstract: Neural
  network abstraction for accelerating verification. In: ATVA, 2020, to appear.

\bibitem{DBLP:conf/cav/PulinaT10}
Pulina, L., Tacchella, A.: An abstraction-refinement approach to verification
  of artificial neural networks. In: Computer Aided Verification, 22nd
  International Conference, {CAV} 2010, Edinburgh, UK, July 15-19, 2010.
  Proceedings. pp. 243--257 (2010)

\bibitem{RHK2018}
Ruan, W., Huang, X., Kwiatkowska, M.: Reachability analysis of deep neural
  networks with provable guarantees. In: IJCAI2018. pp. 2651--2659 (2018)

\bibitem{DBLP:conf/ijcai/RuanWSHKK19}
Ruan, W., Wu, M., Sun, Y., Huang, X., Kroening, D., Kwiatkowska, M.: Global
  robustness evaluation of deep neural networks with provable guarantees for
  the hamming distance. In: Kraus, S. (ed.) Proceedings of the Twenty-Eighth
  International Joint Conference on Artificial Intelligence, {IJCAI} 2019,
  Macao, China, August 10-16, 2019. pp. 5944--5952. ijcai.org (2019)

\bibitem{medicalsystem}
Sheikhtaheri, A., Sadoughi, F., Dehaghi, Z.H.: Developing and using expert
  systems and neural networks in medicine: {A} review on benefits and
  challenges. J. Medical Syst.  \textbf{38}(9), ~110 (2014)

\bibitem{alphago}
Silver, D., Huang, A., Maddison, C.J., Guez, A., Sifre, L., van~den Driessche,
  G., Schrittwieser, J., Antonoglou, I., Panneershelvam, V., Lanctot, M.,
  Dieleman, S., Grewe, D., Nham, J., Kalchbrenner, N., Sutskever, I.,
  Lillicrap, T.P., Leach, M., Kavukcuoglu, K., Graepel, T., Hassabis, D.:
  Mastering the game of go with deep neural networks and tree search. Nature
  \textbf{529}(7587),  484--489 (2016)

\bibitem{krelu}
Singh, G., Ganvir, R., P{\"{u}}schel, M., Vechev, M.T.: Beyond the single
  neuron convex barrier for neural network certification. In: Wallach, H.M.,
  Larochelle, H., Beygelzimer, A., d'Alch{\'{e}}{-}Buc, F., Fox, E.B., Garnett,
  R. (eds.) Advances in Neural Information Processing Systems 32: Annual
  Conference on Neural Information Processing Systems 2019, NeurIPS 2019, 8-14
  December 2019, Vancouver, BC, Canada. pp. 15072--15083 (2019)

\bibitem{deepz}
Singh, G., Gehr, T., Mirman, M., P{\"{u}}schel, M., Vechev, M.T.: Fast and
  effective robustness certification. In: Advances in Neural Information
  Processing Systems 31: Annual Conference on Neural Information Processing
  Systems 2018, NeurIPS 2018, 3-8 December 2018, Montr{\'{e}}al, Canada. pp.
  10825--10836 (2018)

\bibitem{deeppoly}
Singh, G., Gehr, T., P{\"{u}}schel, M., Vechev, M.T.: An abstract domain for
  certifying neural networks. {PACMPL}  \textbf{3}({POPL}),  41:1--41:30 (2019)

\bibitem{fastpolyhedra17}
Singh, G., P{\"{u}}schel, M., Vechev, M.T.: Fast polyhedra abstract domain. In:
  Castagna, G., Gordon, A.D. (eds.) Proceedings of the 44th {ACM} {SIGPLAN}
  Symposium on Principles of Programming Languages, {POPL} 2017, Paris, France,
  January 18-20, 2017. pp. 46--59. {ACM} (2017)

\bibitem{SZSBEGF2014}
Szegedy, C., Zaremba, W., Sutskever, I., Bruna, J., Erhan, D., Goodfellow, I.,
  Fergus, R.: Intriguing properties of neural networks. In: International
  Conference on Learning Representations (ICLR2014) (2014)

\bibitem{imagestar20}
Tran, H., Bak, S., Xiang, W., Johnson, T.T.: Verification of deep convolutional
  neural networks using imagestars. In: Lahiri, S.K., Wang, C. (eds.) Computer
  Aided Verification - 32nd International Conference, {CAV} 2020, Los Angeles,
  CA, USA, July 21-24, 2020, Proceedings, Part {I}. Lecture Notes in Computer
  Science, vol. 12224, pp. 18--42. Springer (2020)

\bibitem{starset19}
Tran, H., Lopez, D.M., Musau, P., Yang, X., Nguyen, L.V., Xiang, W., Johnson,
  T.T.: Star-based reachability analysis of deep neural networks. In: ter Beek,
  M.H., McIver, A., Oliveira, J.N. (eds.) Formal Methods - The Next 30 Years -
  Third World Congress, {FM} 2019, Porto, Portugal, October 7-11, 2019,
  Proceedings. Lecture Notes in Computer Science, vol. 11800, pp. 670--686.
  Springer (2019)

\bibitem{selfdriving}
Urmson, C., Whittaker, W.: Self-driving cars and the urban challenge. {IEEE}
  Intell. Syst.  \textbf{23}(2),  66--68 (2008)

\bibitem{detection2}
Wang, J., Dong, G., Sun, J., Wang, X., Zhang, P.: Adversarial sample detection
  for deep neural network through model mutation testing. In: 2019 IEEE/ACM
  41st International Conference on Software Engineering (ICSE). pp. 1245--1256.
  IEEE (2019)

\bibitem{detection1}
Wang, J., Sun, J., Zhang, P., Wang, X.: Detecting adversarial samples for deep
  neural networks through mutation testing. CoRR  \textbf{abs/1805.05010}
  (2018), \url{http://arxiv.org/abs/1805.05010}

\bibitem{statistical19}
Webb, S., Rainforth, T., Teh, Y.W., Kumar, M.P.: A statistical approach to
  assessing neural network robustness. In: 7th International Conference on
  Learning Representations, {ICLR} 2019, New Orleans, LA, USA, May 6-9, 2019.
  OpenReview.net (2019)

\bibitem{DBLP:conf/icml/WengCNSBOD19}
Weng, L., Chen, P., Nguyen, L.M., Squillante, M.S., Boopathy, A., Oseledets,
  I.V., Daniel, L.: {PROVEN:} verifying robustness of neural networks with a
  probabilistic approach. In: Chaudhuri, K., Salakhutdinov, R. (eds.)
  Proceedings of the 36th International Conference on Machine Learning, {ICML}
  2019, 9-15 June 2019, Long Beach, California, {USA}. Proceedings of Machine
  Learning Research, vol.~97, pp. 6727--6736. {PMLR} (2019)

\bibitem{fastlin}
{Weng}, T.W., {Zhang}, H., {Chen}, H., {Song}, Z., {Hsieh}, C.J., {Boning}, D.,
  {Dhillon}, I.S., {Daniel}, L.: {Towards Fast Computation of Certified
  Robustness for ReLU Networks}. In: ICML 2018 (Apr 2018)

\bibitem{DBLP:conf/tacas/WickerHK18}
Wicker, M., Huang, X., Kwiatkowska, M.: Feature-guided black-box safety testing
  of deep neural networks. In: Beyer, D., Huisman, M. (eds.) Tools and
  Algorithms for the Construction and Analysis of Systems - 24th International
  Conference, {TACAS} 2018, Held as Part of the European Joint Conferences on
  Theory and Practice of Software, {ETAPS} 2018, Thessaloniki, Greece, April
  14-20, 2018, Proceedings, Part {I}. Lecture Notes in Computer Science, vol.
  10805, pp. 408--426. Springer (2018)

\bibitem{DBLP:journals/tcs/WuWRHK20}
Wu, M., Wicker, M., Ruan, W., Huang, X., Kwiatkowska, M.: A game-based
  approximate verification of deep neural networks with provable guarantees.
  Theor. Comput. Sci.  \textbf{807},  298--329 (2020)

\end{thebibliography}

\newpage
\appendix
\section{Proof of Lemma~\ref{lemma:CPO}}

\begin{proof}
We prove the lemma by induction on the dimension $n$. The case for $n=1$ is trivial. 
Now we assume that it holds for $n-1$. 
For an abstract element $a=(a^\le,a^\ge) \in \mathcal A_n$, we can always write it as $a=(a_{1..n-1},a_n)$ where $a_{1..n-1}\in \mathcal{A}_{n-1}$ is the abstract element of the first $n-1$ dimensions, and $a_n=(a_n^\le,a_n^\ge)$. 
Because $\{a^{(k)}\}$ is decreasing, $\{a_{1..n-1}^{(k)}\}$ is also decreasing, and from the induction hypothesis, there exists $a_{1..n-1}^* \in \mathcal{A}_{n-1}$, s.t. $\gamma(a_{1..n-1}^*)=\bigcap_{k=1}^\infty \gamma(a_{1..n-1}^{(k)})$. 
It is easy to see that $(a_{1..n-1}^*,a_n^{(k)})$ is also decreasing and bounded.
For $a_n^{(k)}=(a_n^{(k),\le},a_n^{(k),\ge})$, we write $a_n^{(k),\le}=\sum_{i=1}^{n-1} w_i^{(k)} x_i +b^{(k)}$, and $a_n^{(k),\le}$ is bounded and increasing on $\gamma(a_{1..n-1}^*)$. Because $\{w_i^{(k)}\}$ and $\{b^{(k)}\}$ are bounded, by Bolzano-Weierstrass Theorem, there exists a subsequence, still denoted by $\{w_i^{(k)}\}$, such that these $\{w_i^{(k)}\}$  converge to some $w_i^{*} \in \mathbb R$ and $b^{(k)} \to b^* \in \mathbb R $ as $k \to \infty$.
We set 
$$
a_n^{*,\le}=\sum_{i=1}^{n-1}  w_i^{*} x_i +b^*.
$$
Then $a_n^{(k),\le}$ converges increasingly to $a_n^{*,\le}$ on $\gamma(a_{1..n-1}^*)$ as $k \to \infty$. For $a_n^{(k),\ge}$, we follow a similar procedure as above to obtain $a_n^{*,\ge}$. 
Now we claim that $a^*:=(a_{1..n-1}^*,(a_n^{*,\le},a_n^{*,\ge}))$ satisfies $\gamma(a^*)=\bigcap_{k=1}^\infty \gamma(a^{(k)})$. First we prove that $\gamma(a^*)$ is the limit of the subsequence.
\begin{itemize}
    \item For any $x \in \gamma(a^*)$ and $k$, from the construction of $a^*$, we have $x_{1..n-1} \in \gamma(a_{1..n}^*) \subseteq \gamma(a_{1..n}^{k})$, and
    $$
    x_n \ge a_n^{*,\le}(x_{1..n-1})=\sum_{i=1} ^{n-1}  w_i^{*} x_i +b^* \ge \sum_{i=1}^{n-1} w_i^{(k)} x_i +b^{(k)}=a_n^{(k),\le}.
    $$
    Similarly we have $x_n \le a_n^{(k),\ge}$, and we obtain $x \in \gamma(a^{(k)})$. Immediately we have $x \in \bigcap_k \gamma(a^{(k)})$ since $k$ is arbitrary.
    \item For any $x \in \bigcap_k \gamma(a^{(k)})$, we have $x_{1..n-1} \in \bigcap_k \gamma(a_{1..n-1}^{(k)})=a_{1..n-1}^*$, and
    $$
    x_n \ge a_n^{(k),\le} = \sum_{i=1}^{n-1} w_i^{(k)} x_i +b^{(k)}.
    $$
    By letting $k \to \infty$, we have $x_n \ge a_n^{*,\le}$. Similarly we have $x_n \le a_n^{*,\ge}$, so $x \in \gamma(a^*)$.
\end{itemize}
Thus we have $\gamma(a^*)=\bigcap_k \gamma(a^{(k)})$ for the subsequence. For the original sequence, its limit exists, so it must be equal to the limit of its subsequence, i.e. $\gamma(a^*)=\bigcap_{k=1}^\infty \gamma(a^{(k)})$. We complete the proof. \hfill \qed
\end{proof}

\end{document}